\documentclass{article}
\PassOptionsToPackage{numbers}{natbib} %

\usepackage[main,final]{neurips_2026}

\usepackage[utf8]{inputenc} %
\usepackage[T1]{fontenc}    %
\usepackage[hidelinks,hypertexnames=false]{hyperref} %
\usepackage{url}            %
\usepackage{booktabs}       %
\usepackage{amsfonts}       %
\usepackage{nicefrac}       %
\usepackage{microtype}      %
\usepackage{xcolor}         %

\usepackage[ruled, linesnumbered, algo2e]{algorithm2e}
\usepackage{enumitem}
\usepackage{amsthm, thm-restate}
\usepackage{mathtools}
\theoremstyle{plain}
\newtheorem{theorem}{Theorem}[section]

\newtheorem{corollary}[theorem]{Corollary}
\theoremstyle{definition}
\newtheorem{definition}[theorem]{Definition}

\theoremstyle{remark}

\def \A {\mathcal{A}}

\def \D {\mathcal{D}}

\def \G {\mathcal{G}}

\def \I {\mathcal{I}}

\def \M {\mathcal{M}}
\def \O {\mathcal{O}}
\def \P {\mathcal{P}}
\def \Q {\mathcal{Q}}

\def \T {\mathcal{T}}

\def \X {\mathcal{X}}

\def \Z {\mathcal{Z}}
\def \eps {\epsilon}
\def \del {\delta}
\def \lam {\lambda}

\def \nat {\mathbb{N}}
\def \rr {\mathbb{R}}
\def \bb {\mathbb{B}}

\DeclareMathOperator*{\expec}{\mathbb{E}}

\DeclareMathOperator*{\argmin}{arg\,min}
\newcommand{\innp}[1]{\left\langle #1 \right\rangle}
\newcommand{\cbra}[1]{\left\{ #1 \right\}}

\newcommand{\rbra}[1]{\left( #1 \right)}
\newcommand{\ceil}[1]{\left\lceil #1 \right\rceil}
\newcommand{\floor}[1]{\left\lfloor #1 \right\rfloor}

\newcommand\numberthis{\addtocounter{equation}{1}\tag{\theequation}}

\allowdisplaybreaks
\numberwithin{equation}{section}

\ifx\example\undefined
\newtheorem*{example}{Example}
\fi

\def \sign {\mathsf{sign}}
\def \estimate {\mathsf{IP}_\mathsf{est}}
\def \replay {\mathsf{Replay}}
\def \aq {\mathsf{AQ}}
\def \aqals {{\mathsf{AQA}\textup{-}\mathsf{OGD}}}
\newcommand\iidsim{\stackrel{\mathclap{iid}}{\,\sim\,}}
\def \tq {\tilde{q}}

\def \keys {\mathsf{Keys}}
\def \get {\mathsf{Get}}
\def \remove {\mathsf{Remove}}

\def \uhat {\hat{u}}

\def \negl {\mathsf{negl}}
\def \iso {\mathsf{iso}}
\def \pso {\mathsf{PSO}}
\def \succ {\mathsf{Succ}}
\def \low {\mathsf{low}}
\def \high {\mathsf{high}}
\def \poly {\mathsf{poly}}
\def \DSQ {\mathsf{DSQ}}

\newcounter{gamectr}
\newenvironment{game}[1][]{
  \SetAlgorithmName{Game}{Game}{List of Games}%
  \refstepcounter{gamectr}%
  \let\oldthealgocf\thealgocf%
  \renewcommand{\thealgocf}{\thegamectr}%
  \begin{algorithm2e}[#1]
}{
  \end{algorithm2e}
  \addtocounter{algocf}{-1}%
  \renewcommand{\thealgocf}{\oldthealgocf}%
}

\title{Efficient Adaptive Data Analysis over \\Dense Data Distributions}

\author{Joon Suk Huh \\
Department of Computer Science\\
University of Wisconsin–Madison, USA \\
\texttt{joon@cs.wisc.edu}
}

\begin{document}

\maketitle

\begin{abstract}
Modern data workflows are inherently adaptive, repeatedly querying the same dataset to refine and validate sequential decisions, but such adaptivity can lead to overfitting and invalid statistical inference. Adaptive Data Analysis (ADA) mechanisms address this challenge; however, there is a fundamental tension between computational efficiency and sample complexity. For $T$ rounds of adaptive analysis, computationally efficient algorithms typically incur suboptimal $\O(\sqrt{T})$ sample complexity, whereas for arbitrary distributions, algorithms with $\O(\log T)$ sample complexity are computationally intractable under standard cryptographic assumptions. In this work, we shed light on this trade-off by identifying a natural class of data distributions under which both computational efficiency and logarithmic sample complexity in $T$ are achievable. We propose a computationally efficient ADA mechanism that attains $\O(\log T)$ sample complexity, optimal in its dependence on $T$, when the data distribution is dense with respect to a known prior. This setting includes, in particular, feature--label data distributions arising in distribution-specific learning. As a consequence, our mechanism also yields a sample-efficient (i.e., $\O(\log T)$ samples) statistical query oracle in the distribution-specific setting. Moreover, although our algorithm is not based on differential privacy, it satisfies a relaxed privacy notion known as Predicate Singling Out (PSO) security~\citep{cohen2020towards} whenever the number of its possible update transcripts is polynomial in the sample size. Our results thus reveal an inherent connection between adaptive data analysis and privacy beyond differential privacy.
\end{abstract}
\section{Introduction}\label{sec:intro}
The inherent adaptivity prevalent in modern data science workflows, while undeniably powerful for discovery and iterative refinement, introduces significant challenges to statistical validity. The repeated interaction with and reuse of the same dataset for both hypothesis generation and subsequent validation can lead to \emph{overfitting} or the generation of \emph{false discoveries} (often colloquially termed ``p-hacking'' or attributed to ``researcher degrees of freedom''). Such issues result in findings that may appear statistically significant within the analyzed sample but fail to generalize reliably to the underlying population from which the data were drawn. This breakdown of classical statistical guarantees contributes directly to the reproducibility crisis observed across empirical scientific fields~\citep{ioannidis2005most,gelman2014statistical}. The core problem, often termed ``double-dipping'', arises when the same data is used first to inform the analysis strategy and then again to validate the resulting conclusions, thereby coupling the randomness in the inferential target with the randomness in the data used for inference.

Adaptive data analysis (ADA)~\citep{dwork2015preserving} is a modern algorithmic framework aimed at enabling statistically valid inference in adaptive settings, using as few samples as possible. At its core is the design of an intermediary, called an ADA mechanism, that mediates between the data and the analyst, providing answers to adaptively chosen queries while ensuring statistical validity. Since the foundational work of~\cite{hardt2014preventing,dwork2015preserving}, there has been extensive progress in developing algorithmic techniques for preserving validity in adaptive settings~\citep{dwork2015generalization,dwork2015reusable,cummings2016adaptive,russo2016controlling,rogers2016max,bassily2016algorithmic,smith2017information,feldman2017generalization,feldman2018calibrating,de2019rademacher,zrnic2019natural,jung2020new,fish2020sampling,rogers2020guaranteed,kontorovich2022adaptive,dagan2022bounded,marchant2024adaptive,blanc2025subsampling}, primarily based on Differential Privacy (DP)~\citep{dwork2014algorithmic} and/or algorithmic stability~\citep{bassily2016algorithmic,smith2017information,shenfeld2019necessary,shenfeld2023generalization}.

In the standard setting, an analyst asks linear queries $q(x)\in\rr$ and seeks to estimate expectations of the form $\expec_{x\sim \D}[q(x)]$. When the queries are fully adaptive, existing computationally efficient approaches achieve $\O(\sqrt{T})$ sample complexity, where $T$ is the number of query rounds. This rate is tight for \emph{general} data distributions: an $\Omega(\sqrt{T})$ lower bound holds for any \emph{computationally efficient} ADA mechanism under standard cryptographic assumptions, such as the existence of one-way functions~\citep{steinke2015interactive, steinke2016between, bun2017make, ullman2018limits, nissim2023adaptive, dinur2023differential}. 

In contrast, data sketching-based methods such as Private Multiplicative Weights (PMW), which maintain an internal sketch of the dataset~\citep{hardt2010multiplicative}, can achieve $\O(\log T)$ sample complexity, but are computationally intractable, as maintaining such a sketch often requires time exponential in the data dimension $n$.

Hence, there is a general tension between computational efficiency and sample efficiency in adaptive data analysis. However, existing hardness results that rule out efficient $o(\sqrt{T})$ sample complexity assume that the mechanism must work for \emph{arbitrary} data distributions. This observation raises a natural question: \emph{Is there a natural class of data distributions $\D$ that admits both $\O(\log T)$ sample complexity and computational efficiency?}

We answer this question in the affirmative. Specifically, we show that for any sufficiently well-spread (``dense'') distribution with respect to a known prior (which may be uniform), there exists an ADA mechanism that is both computationally efficient and sample efficient, with sample complexity growing only logarithmically in the number of queries. Such distributions arise naturally in distribution-specific learning~\citep{kharitonov1993cryptographic,balcan2013active,shamir2018distribution,blanc2025distributional}, including distributions over pairs $(x,f(x))$ where $x$ is drawn from a known marginal distribution over features and $f$ is an unknown Boolean function.

\subsection{Main Results}
\paragraph{Efficient ADA Mechanism for Dense Distributions.}
We first present a computationally efficient ADA mechanism for dense data distributions. Density is defined with respect to a known reference (prior) distribution. Specifically, we assume that both the mechanism and the analyst are given a prior distribution $\D_g$ generated by a known, efficiently computable function $g:\{0,1\}^n \rightarrow \I$ over a finite domain $\I$. We then consider an \emph{unknown} target distribution $\D$ that is $\lam$-dense in $\D_g$, meaning that $\Pr_{I\sim\D}[I=i]\leq\frac{1}{\lam}\Pr_{I\sim\D_g}[I=i]$ for all $i \in \I$.

Intuitively, this condition implies that one could sample from $\D$ using $\D_g$ via rejection sampling with acceptance probability at least $\lam$. However, this form of ``rejection samplability'' does not imply that the analyst can perform data analysis over the unknown distribution $\D$, since the required acceptance rule depends on the unknown target. Such an unknown rule arises, for example, when a filter retains records of a known population with average acceptance probability at least $\lam$; the retained records then follow a $\lam$-dense distribution.

\begin{theorem}[Restatement of Theorem~\ref{thm:AQALS-CC} and Corollary~\ref{cor:dense}]
    Let $g:\{0,1\}^n \rightarrow \I$ be an efficiently computable function. There exists an ADA mechanism, depending only on $g$, $\lam$, $\eps$, $\del$, and $T$, with the following guarantee. For any distribution $\D$ over $\I$ that is $\lam$-dense with respect to the distribution $\D_g$ generated by $g$, the mechanism answers $\expec_{k\sim\D}[q_t(k)]$ to within $\eps$ accuracy for any sequence of $T$ adaptively chosen queries $(q_t)_{t\in[T]}$, where each $q_t:\I\rightarrow\{-1,1\}$, with probability at least $1-\del$, using $M\in\O(\lam^{-1}\eps^{-4}\log(T/\eps)+\eps^{-2}\log(1/\del))$ i.i.d. samples from $\D$. For queries evaluable in $\poly(n)$ time, the runtime is $\poly(n,\eps^{-1},\lam^{-1},T,\log(1/\del))$.
\end{theorem}
Notably, the sample complexity is independent of the ambient data dimension $n$, and its logarithmic dependence on $T$ is optimal even for non-adaptive queries at constant density (Proposition~\ref{prop:lower-bound}). More generally, if some surrogate $\D'$ matches the analyst's queries within $\eps/16$ and its likelihood ratio $w=d\D'/d\D_g$ satisfies $\expec_{\D_g}w^2\leq\rho^2$, the sample complexity is $\O(\rho^2\eps^{-4}\log(T/\eps)+\eps^{-2}\log(1/\del))$; under pointwise $\lam$-density, $\expec_{\D_g}w^2\leq1/\lam$, so $\rho=\lam^{-1/2}$ suffices.

\paragraph{Sample-efficient Distribution-Specific Statistical Query Oracle.}
Why, then, do we care about dense target distributions? A natural example arises in distribution-specific learning. Let $X\sim\D_h$ for a known, efficiently samplable feature marginal and $Y=f(X)\in\{-1,1\}$ for an unknown labeling function. This labeled distribution is $1/2$-dense with respect to the prior that pairs $X$ with an independent uniform label. A known feature sampler of this kind arises, for example, in experimental designs and simulations where features come from a known process and outcomes are unknown.

Our mechanism yields a distribution-specific Statistical Query (SQ) oracle~\citep{kearns1998efficient} using one labeled sample for all adaptive calls (Theorem~\ref{thm:AQALS-DSQ}). Bounded losses, gradient coordinates, and conditional label statistics reduce to such queries; in \S\ref{subsec:dsq} we give the reduction and compare the resulting oracle with classical SQ simulation.

\paragraph{Privacy Guarantee.}
Many existing ADA mechanisms are built on differential privacy (DP) and thus often inherit formal DP guarantees. In contrast, our mechanism is not based on DP and does not provide DP guarantees. Nevertheless, perhaps surprisingly, it still offers a meaningful privacy guarantee in terms of \emph{Predicate Singling Out} (PSO) security~\citep{cohen2020towards}, a notion that is strictly weaker than DP and was designed to capture the ``singling out'' risk addressed by the European Union’s General Data Protection Regulation (GDPR). Specifically, the success probability of a PSO adversary is at most the number of possible update transcripts times the largest success probability of an adversary that sees no output, so PSO security holds when this number is polynomial in the sample size, for example at constant accuracy and density with $T=\poly(M)$ (Theorem~\ref{thm:AQALS-PSO}). This bound follows from an argument that builds on a key lemma from our sample complexity analysis and, notably, requires no explicit noise addition.

\subsection{Technical Overview}
Unlike many computationally efficient ADA mechanisms that achieve $\O(\sqrt{T})$ sample complexity, our mechanism does not rely on differential privacy or algorithmic stability arguments. Instead, it builds on an iterative data sketching strategy inspired by the Private Multiplicative Weights (PMW) algorithm~\citep{hardt2010multiplicative}. The goal is to maintain an internal sketch of the dataset $S\in\I^M$ such that linear statistical queries evaluated on the sketch closely approximate their values on $S$. While PMW achieves this by maintaining a distribution over the data universe and updating it via multiplicative weights, this approach incurs computational cost proportional to the domain size $2^n$, which is intractable.

We avoid this intractability by using (unconstrained) Online Gradient Descent (OGD)~\citep{zinkevich2003online} in the seed space $\rr^N$, where $N=2^n$. Through the known prior, we represent a surrogate $\D'$ with likelihood ratio $w=d\D'/d\D_g$ and $\expec_{\D_g}w^2\leq\rho^2$ as the vector with coordinates $w(g(r))/(\rho\sqrt N)$, which lies in the Euclidean unit ball and is the comparator in our analysis. Specifically, our algorithm can be viewed as a lazy version of OGD, where an update is performed only when the approximation error exceeds a prescribed threshold. This lazy update rule, together with the standard regret bound for OGD, implies that the number of such updates, and hence the number of vectors stored in the sketch, is bounded by $\O(\rho^2\eps^{-2})$. Crucially, the intractability of PMW is avoided because each stored vector is a scaled query function, so answers can be reconstructed efficiently from the sketch via randomized inner-product estimation at uniformly random seeds.

Equivalently, the current sketch minimizes the sum of the recorded linearized losses plus a fixed quadratic regularizer, so the mechanism can also be viewed as follow-the-regularized-leader (FTRL)~\citep{mcmahan2017survey}; Equation~\eqref{eq:FTRL} states the exact identity.

The bounded number of updates (mistakes) of order $\O(\rho^2\eps^{-2})$ in constructing the internal sketch of the dataset implies the existence of a succinct transcript of the mechanism's update behavior, of size $\O(\rho^2\eps^{-2}\log(T/\eps))$ bits. As this transcript can simulate the mechanism's exact behavior, it can be used to reconstruct all queries asked by $\A$, using only the internal randomness of the mechanism and the analyst, and without access to the sample $S$. This transcript plays a central role in ensuring statistical validity. Since it determines the entire query sequence, we can apply standard concentration arguments, namely Hoeffding's inequality and the union bound, over the small space of possible transcripts. This yields uniform convergence, guaranteeing that $|q_t(S) - q_t(\D)| \in \O(\eps)$ for all $t \in [T]$ with high probability, with sample size $|S| \in \O(\rho^2\eps^{-4}\log(T/\eps)+\eps^{-2}\log(1/\del))$.

\subsection{Related Work}
Many algorithmic frameworks for adaptive data analysis (ADA) are based on differential privacy or algorithmic stability, as studied in~\citep{bassily2016algorithmic,smith2017information,feldman2018calibrating,shenfeld2019necessary,jung2020new,blanc2025subsampling}. These approaches provide strong generalization guarantees under adaptivity and run efficiently, but typically incur a sample complexity of $\O(\sqrt{T})$.

Closer to our approach, the idea of answering queries by reconstructing (or sketching) a dataset was introduced in~\cite{blum2013learning}, which used the exponential mechanism~\citep{mcsherry2007mechanism} both to reconstruct the dataset and to privately answer \emph{non-adaptive} queries. This was extended in the seminal work of~\cite{hardt2010multiplicative}, which proposed the Private Multiplicative Weights (PMW) algorithm, based on the multiplicative weights framework~\citep{freund1999adaptive,arora2012multiplicative}, to handle adaptive queries while satisfying differential privacy. Notably, while PMW is designed as a differentially private algorithm, it also serves as an ADA mechanism with $\O(\log T)$ sample complexity and was later extended beyond linear queries~\citep{ullman2015private}. However, it is computationally intractable in general, as it requires sampling from Gibbs distributions over exponentially large domains like $\{0,1\}^n$.

Hardt and Rothblum~\cite{hardt2010multiplicative} also give an efficient variant of PMW for pseudo-smooth databases. A $\xi$-smooth histogram puts mass at most $\xi$ on each point, so smoothness is $\lam$-density relative to the uniform distribution with $\lam=1/(\xi N)$; pseudo-smoothness relative to a query class parallels our pseudodensity. The accuracy guarantee of this variant~\citep[Theorem~I.2]{hardt2010multiplicative} holds only for query sequences fixed in advance. Their general PMW mechanism allows adaptive queries but runs in time nearly linear in the universe size. Our result provides population accuracy for adaptive queries under a shared efficient prior sampler, in time polynomial in $n$. For distribution-specific SQ algorithms, we compare our implementation with classical SQ simulation~\citep{kearns1998efficient} in \S\ref{subsec:dsq}.

\subsection{Organization}
In \S\ref{sec:preliminaries}, we introduce the adaptive data analysis framework and define (pseudo)dense data distributions. In \S\ref{sec:algorithm}, we present our efficient ADA mechanism, denoted $\M_\aqals$. In \S\ref{sec:analysis}, we establish the correctness and sample complexity guarantees of $\M_\aqals$. In \S\ref{sec:privacy}, we introduce Predicate Singling Out (PSO) and bound the success probability of PSO attacks. We discuss limitations in \S\ref{sec:discussion}.

\section{Preliminaries}\label{sec:preliminaries}
In this section, we provide the necessary background on adaptive data analysis and introduce key definitions used throughout the paper.

\paragraph{Notation.}
We denote the set $\{1, \dots, N\}$ by $[N]$. The sign function $\sign(x)$ returns $+1$ if $x \geq 0$ and $-1$ otherwise. The set of non-negative real numbers is denoted by $\rr_+$. For any vector $v \in \rr^N$, we write $\|v\|$ to denote its Euclidean norm. The $N$-dimensional Euclidean unit ball is defined as $\bb^N \coloneq \{v \in \mathbb{R}^N : \|v\| \leq 1\}$. For any finite set $S$, $U_S$ denotes the uniform distribution over $S$. Finally, for any function $q : \I \rightarrow \rr$ and any distribution $\D$ over $\I$, we define $q(\D) \coloneq \expec_{i \sim \D}[q(i)]$ as the expectation of $q$ under $\D$. For a multiset $S=(I_1,\ldots,I_M)$ of records in $\I$, we write $q(S)\coloneq M^{-1}\sum_{j=1}^M q(I_j)$, counting repeated records with their multiplicities. For the sake of presentation, we identify the $n$-bit Boolean cube $\{0,1\}^n$ with $[N]$, where $N = 2^n$, throughout the paper. We take $T\geq1$, $0<\eps,\del\leq1/2$, and $0<\lam\leq1$.

\begin{game}[t]
\caption{$\aq_{N,M,T}[\A \rightleftharpoons \M]$: Adaptive Query Game} \label{game:aq}
\DontPrintSemicolon
$\A$ chooses a distribution $\D$ over $\I$ unknown to $\M$.\\
A multiset $S$ of $M$ i.i.d. samples from $\D$ is given to $\M$ but not to $\A$.\\
\vspace{0.5em}
\For{$t \in [T]$}{
    $\A$ chooses a query $q_t \in \{-1,1\}^{\I}$ given $(q_1, a_1, \dots, q_{t-1}, a_{t-1})$.\\
    $\M$ outputs an answer $a_t \in [-1,1]$ given $(q_1, a_1, \dots, q_{t-1}, a_{t-1}, q_t)$.\\
}
\end{game}

\subsection{Problem Setting: Adaptive Data Analysis}
An adaptive data analysis (ADA) mechanism $\M$ is an algorithm that holds a sufficient amount of i.i.d. data $S=(I_1,\ldots,I_M)$ drawn from an \emph{unknown} distribution $\D$ over a finite domain $\I$, and provides answers $a_t$ to a sequence of adaptively chosen Boolean queries $(q_t)_{t \in [T]}$, where each $q_t : \I \rightarrow \{-1, 1\}$, selected by an analyst (or adversary) $\A$. The goal is to ensure that these answers are $\eps$-accurate; that is, $|a_t - q_t(\D)| \leq \eps$ for all $t \in [T]$.

In the non-adaptive setting, where all queries are chosen in advance, it suffices for $\M$ to hold $|S| \in \O(\log T)$ samples, by standard concentration arguments (e.g., Hoeffding’s inequality and the union bound). Our goal is to retain this $\O(\log T)$ sample complexity even when queries are chosen adaptively, under certain assumptions on the data distribution, while ensuring that the ADA mechanism remains \emph{computationally efficient}; that is, it runs in time $\poly(\eps^{-1},\lam^{-1},\log N,|S|,T,\log(1/\del))$, with $\lam^{-1}$ replaced by $\rho$ under the bounded-$L_2$ condition below.

We formalize adaptive data analysis as a game between the ADA mechanism $\M$ and the analyst (or adversary) $\A$, referred to as the \emph{Adaptive Query Game}, depicted in Game~\ref{game:aq}. Note that allowing $\A$ to choose the distribution $\D$ in Game~\ref{game:aq} only strengthens $\A$; thus, our result also applies to the more natural setting in which $\D$ is \emph{unknown} to both the analyst $\A$ and the mechanism $\M$.

Let $\Pr_{\aq_{N,M,T}[\A \rightleftharpoons \M]}$ denote the probability distribution over all random variables (such as the $a_t$s and $q_t$s) that are realized during the execution of the game $\aq_{N,M,T}[\A \rightleftharpoons \M]$. We then define  accuracy of $\M$ as follows:
\begin{definition}[Accuracy]
A mechanism $\M$ is said to be $(\eps, \del)$-accurate with respect to $\A$ with sample complexity $M_0$ if, for all $M \geq M_0$,
\begin{align*}
    \Pr_{\aq_{N,M,T}[\A \rightleftharpoons \M]}\Big[\forall t\in [T],\ \left| a_t-q_t(\D) \right|\leq\eps \Big]\geq 1-\del,
\end{align*}
where the probability is taken over $S\sim\D^M$ and the internal randomness of both $\M$ and $\A$.
\label{def:accruacy}
\end{definition}

\paragraph{Dense Data Distribution.} 
We assume that both the analyst and the mechanism have access to a function $g:[N]\rightarrow\I$ for some finite domain $\I$, and hence can sample from the distribution $\D_g$, where $I\sim\D_g$ is given by $I\coloneq g(X)$ with $X\sim U_{[N]}$. We further assume that $g$ is efficiently computable, namely in $\poly(\log N)$ time. We say that a distribution $\D$ over $\I$ is \emph{$(\lam,\D_g)$-dense} if $\Pr_{I \sim \D}[I = i] \leq \frac{1}{\lam} \Pr_{I \sim \D_g}[I = i]$ for all $i \in \I$.

\paragraph{Examples of Dense Distributions.}
If a record $z\sim\D_g$ is accepted with an unknown probability $a(z)\in[0,1]$ and $\expec_{\D_g}a\geq\lam$, the distribution of accepted records satisfies $\D(z)/\D_g(z)=a(z)/\expec_{\D_g}a\leq1/\lam$. This includes conditioning on an event of prior probability at least $\lam$. Bounded reweighting, $\D=w\D_g$ with $0\leq w\leq1/\lam$ and $\expec_{\D_g}w=1$, gives the same condition. A component of a known mixture $\D_g=\alpha\D+(1-\alpha)Q$ is also $\lam$-dense whenever $\alpha\geq\lam$. The mechanism needs the shared sampler and a bound on the density parameter; it does not need either probability mass function, the acceptance rule, or a certificate of density.

Our algorithm, in fact, applies to a slightly broader class of distributions that includes not only truly dense distributions but also \emph{pseudodense} ones, namely distributions that are indistinguishable from dense distributions from the perspective of the analyst. We formalize this notion below.
\begin{definition}[Pseudodensity]
    For any adversary $\A$, let $\Q^\A \subseteq \{-1,1\}^{\I}$ denote the set of all queries that $\A$ may output at any point, on any input, and under any internal randomness. We say that a distribution $\D$ over $\I$ is $(\eps, \lam,\D_g)$-\emph{pseudodense} with respect to $\A$ if there exists a $(\lam,\D_g)$-dense distribution $\D'$ over $\I$ such that $|q(\D) - q(\D')| \leq \eps$ for all $q \in \Q^\A$. Moreover, if $\A$ almost surely selects a distribution $\D$ that is $(\eps, \lam,\D_g)$-pseudodense with respect to itself, we say that $\A$ is an $(\eps, \lam,\D_g)$-\emph{pseudodense challenger}.
    \label{def:pseudodensity}
\end{definition}

Pseudodensity is relative to all queries in $\Q^\A$, including those on histories that are not realized. For example, if queries on records $(u,v)$ depend only on $u$, a target with the same $u$-marginal as the prior is $(0,1,\D_g)$-pseudodense, even if its conditional distribution of $v$ is concentrated. For an unrestricted Boolean query family, the condition instead requires $2\operatorname{TV}(\D,\D')\leq\eps$.

\paragraph{A Bounded-$L_2$ Formulation.}
More generally, the analysis only needs a surrogate $\D'\ll\D_g$ whose likelihood ratio $w=d\D'/d\D_g$ satisfies, for a given bound $\rho\geq1$,
\begin{equation}
    \expec_{\D_g}[w^2]\leq\rho^2,
    \qquad |q(\D)-q(\D')|\leq\eps/16\quad\text{for all }q\in\Q^\A.
    \label{eq:L2-condition}
\end{equation}
Pointwise $\lam$-density implies $\expec_{\D_g}[w^2]\leq\lam^{-1}\expec_{\D_g}[w]=\lam^{-1}$. Thus, $(\eps/16,\lam,\D_g)$-pseudodensity satisfies~\eqref{eq:L2-condition} with $\rho=\lam^{-1/2}$. We use the scale $\kappa=1/\rho$ in the algorithm. Neither $w$ nor $\D'$ is an input, and $\D$ itself need not be absolutely continuous with respect to $\D_g$.

\paragraph{Working in the Seed Space.}
The prior generator $g$ is fixed throughout the AQ game. For each query $q_t:\I\to\{-1,1\}$, write $\tq_t=q_t\circ g\in\{-1,1\}^N$. The internal sketch uses these vectors in seed space; the sample mean $q_t(S)$ is always evaluated on the observed records in $\I$. This requires evaluating $g$, but never inverting it or recovering seeds for target samples.

\subsection{Application: Implementing Distribution-Specific Statistical Query Oracle}\label{subsec:dsq}
The above setting also captures a standard distribution-specific learning scenario. Consider a distribution $\D_{h,f}$ over labeled examples $(X,Y)$, where $X$ is drawn from a \emph{known} marginal distribution $\D_h$ over a finite feature space $\X$, generated by an efficient sampler $h$, and $Y\coloneq f(X)\in\{-1,1\}$ for an unknown labeling function $f:\X\to\{-1,1\}$. Define a prior distribution $\D_g$ by letting $I\sim\D_g$ be given as $I\coloneq g(R)\coloneq (h(R_{1:n-1}), 2R_n-1)$ for a uniformly random seed $R\sim U_{[N]}$. Equivalently, $\D_g$ samples a feature vector $X\sim \D_h$ together with an independent uniform label in $\{-1,1\}$.

Under this construction, the true labeled distribution $\D$ is $(1/2)$-dense with respect to $\D_g$. Consequently, an ADA mechanism designed for $(\lam,\D_g)$-dense distributions can be interpreted as providing a \emph{distribution-specific} statistical query interface for the labeled distribution $\D_{h,f}$.

\begin{definition}[Distribution-specific SQ oracle~\citep{kearns1998efficient}]
    Fix a known distribution $\D_h$ over a feature space $\X$ and an unknown labeling function $f:\X\to\{-1,1\}$. Let $\D_{h,f}$ denote the induced distribution over $\X\times\{-1,1\}$ obtained by sampling $X\sim \D_h$ and setting $Y\coloneq f(X)$. A distribution-specific statistical query oracle, denoted $\DSQ_h(f,\eps)$, takes as input a query function $\phi:\X\times\{-1,1\}\to [-1,1]$ and returns a value $v$ satisfying
    \begin{align*}
        \left|v-\phi(\D_{h,f})\right|\leq \eps.
    \end{align*}
    An algorithm that interacts with the unknown $f$ only through calls to $\DSQ_h(f,\eps)$ is called a distribution-specific SQ algorithm (with respect to $\D_h$).
\end{definition}
DSQ queries take values in $[-1,1]$, whereas the queries in Game~\ref{game:aq} are Boolean; we handle this mismatch by a standard randomization. For a bounded query $\phi_t$, draw independent thresholds $U_{t,j}\sim\operatorname{Unif}[-1,1]$ and set $q_{t,j}(x,y)=\sign(\phi_t(x,y)-U_{t,j})$, so that $\expec_{U_{t,j}}[q_{t,j}(x,y)]=\phi_t(x,y)$. Each threshold is fixed for all evaluations of its query. Averaging the answers to $B=\Theta(\eps^{-2}\log(T/\del))$ such Boolean queries implements one DSQ call. All $BT$ queries use the same labeled sample; Appendix~\ref{app:proof:thm:AQALS-DSQ} gives the error and confidence calculation.

\paragraph{Motivation for Linear Queries.}
The expectation of a bounded loss, or of a bounded coordinate of its gradient, is a linear query after scaling to $[-1,1]$. Such queries can be issued adaptively by gradient-based optimization methods~\citep{feldman2021statistical}. Each coordinate counts as a separate call. Conditional label statistics also reduce to linear queries: for bounded $\psi$ and a subgroup $A\subseteq\X$,
\[
\expec[\psi(Y)\mid X\in A]
=\frac{\expec[\psi(Y)\mathbf{1}\{X\in A\}]}{\D_h(A)}.
\]
The numerator is a linear query and the denominator can be estimated from the known sampler. The subgroup mass $\D_h(A)$ must be bounded away from zero at the desired accuracy scale, since division amplifies estimation error.

\paragraph{Comparison with SQ Simulation.}
Classical SQ simulation already reuses one sample by uniform convergence over the query class~\citep[Theorem~1]{kearns1998efficient}; the bound depends on the class size or its VC dimension. A fresh-sample implementation instead uses $\O(T\eps^{-2}\log(T/\del))$ labeled examples. Using the known marginal, we control the sample complexity through the number $T$ of adaptive calls, with no complexity term for the query class. Our guarantee still requires efficient query evaluation and sampling, and its dependence on the accuracy is worse than that of uniform convergence.

\begin{algorithm2e}[t]
\caption{$\M_\aqals$: Adaptive Query Answering via Online Gradient Descent}\label{alg:Maqals}
\SetKwComment{Comment}{\hfill\color{blue}$\triangleright$~}{}
\DontPrintSemicolon

\textbf{Oracle:} Adversary $\A$.\\
\vspace{0.25em}
\textbf{Input:} Data $S\in\I^M$, prior generator $g$, accuracy $\eps$, failure probability $\del$, and norm bound $\rho\geq1$; set $\kappa=1/\rho$.\\
\vspace{0.5em}
$\G \gets \emptyset$ \Comment*{{\rmfamily\color{blue}$\G$ is a maintained multiset of gradients.}}
\vspace{0.5em}
\For{$t \in [T]$}{
    \vspace{0.5em}
    Get a query $q_t$ from $\A$ and set $\tq_t=q_t\circ g$.\\
    \vspace{0.25em}
    $\uhat_t \gets \estimate(\tq_t,\G;\kappa \eps / 8, \del/2T)$ \Comment*{{\rmfamily\color{blue}$\estimate$ is given in Alg.~\ref{alg:estimate}}}
    \vspace{0.5em}
    \If{$|\uhat_t - \kappa q_t(S)| > \kappa\eps/2$ \textbf{\textup{and}} $|\G| \leq 128(\kappa\eps)^{-2}$}{
        \vspace{0.5em}
        $s_t \gets \sign(\uhat_t - \kappa q_t(S))$\\
        \vspace{0.5em}
        $\G \gets \G \cup \big\{ -\kappa\eps s_t \tq_t / 16 \big\}$ \Comment*{{\rmfamily\color{blue} Update $\G$.}}
        \vspace{0.5em}
        $a_t \gets\ $ Round $q_t(S)$ to the nearest multiple of $\eps/2$.
        \vspace{0.5em}
    }
    \Else{
        $a_t \gets \uhat_t / \kappa$
    }
    Set $a_t\gets\min\{1,\max\{-1,a_t\}\}$ and send $a_t$ to $\A$.
}
\end{algorithm2e}

\begin{algorithm2e}[t]
\caption{$\estimate(q,\G;\eps,\del)$}
\label{alg:estimate}
\SetKwComment{Comment}{\hfill\color{blue}$\triangleright$~\color{blue}}{}
\DontPrintSemicolon

\textbf{Input:} Query $q \in \{-1,1\}^N$, gradient vectors $\G$, accuracy $\eps > 0$ and failure probability $\del >0$.\\
\vspace{0.25em}
\textbf{Output:} An $\eps$-accurate estimate of $\innp{q, y}/N$, where $y \coloneq \sum_{v \in \G} v$.\\
\vspace{0.25em}
Sample $(i_1, \dots, i_{L}) \iidsim U_{[N]}$ where $L\coloneq\left\lceil 2\eps^{-2}T^2\ln\frac{2T}{\del} \right\rceil$.\\
\vspace{0.25em}
\Return{$\frac{1}{L}\sum_{v\in \G}\sum_{k=1}^{L} q(i_k)v(i_k)$}
\end{algorithm2e}

\section{ADA Mechanism via Online Gradient Descent}\label{sec:algorithm}
In this section, we present our main contribution, the ADA mechanism $\M_\aqals$ (Alg.~\ref{alg:Maqals}).

\paragraph{Overview of $\M_\aqals$.}
Given samples $S\in\I^M$, $\M_\aqals$ aims to answer queries using an approximation (sketch) of $S$ whenever possible, rather than answering with $q_t(S)$ directly for every $t$. To enable this, $\M_\aqals$ maintains a multiset $\G \subseteq \rr^N$ consisting of vectors proportional to previous queries in seed space, $\tq_t \in \{-1,1\}^N$, interpreted as vectors in $\rr^N$. Internally, each of these vectors is proportional to a subgradient of the convex objective function $\hat f_t(x) \coloneq \left|\innp{\tq_t/\sqrt{N}, x} - \kappa q_t(S)\right|$. For this reason, we refer to the set $\G$ as the collection of ``gradient'' vectors.

In each round, at Line~6, $\M_\aqals$ attempts to approximate the desired answer using the $\estimate$ subroutine, defined in Alg.~\ref{alg:estimate}, based on the current collection of gradients $\G$. The desired answer is encoded in $\innp{\tq_t, y_t}/N$ for $y_t\coloneq\sum_{v\in\G}v$, which is designed to approximate $\kappa q_t(S)$; we call $x_t\coloneq y_t/\sqrt N$ the normalized sketch. The subroutine $\estimate$ approximates $\innp{\tq_t, y_t}/N$ by taking the sample mean of $\tq_t(i)y_t(i)$. This randomized estimation avoids the prohibitive cost of full-dimensional computation and ensures that the running time depends on $N$ only through $\poly(\log N)$.

After estimating the tentative answer $\uhat_t/\kappa$ in Line 6, $\M_\aqals$ checks whether this estimate is sufficiently close to $q_t(S)$. If so, it outputs $a_t = \uhat_t/\kappa$; otherwise, it updates $\G$ according to Line 9, effectively performing a gradient step on $\hat f_t$, and returns the answer using $q_t(S)$. We refer to this mechanism as AQA-OGD, as it follows an online gradient descent procedure to maintain an approximate representation of $S$.

The key insight lies in the second condition of Line~7. Although $\M_\aqals$ evaluates $q_t(S)$ in every round, it updates $\G$ at most $\O(\rho^2\eps^{-2})$ times, independent of the total number of queries $T$ or the domain size $N$. As shown in \S\ref{sec:analysis}, this limitation is crucial to our analysis, ensuring that the accuracy of each $a_t$ is preserved via uniform concentration of $q_t(S) \approx q_t(\D)$ for all $t \in [T]$.

\paragraph{Accuracy and Efficiency of $\M_\aqals$.}
We now provide the computational efficiency and sample complexity guarantees for $\M_\aqals$. We begin with its computational efficiency. The construction of $\M_\aqals$ in Algorithms~\ref{alg:Maqals} and~\ref{alg:estimate} immediately yields the following bound.

\begin{restatable}{theorem}{thmAQALSCC}\emph{\textbf{[Computational Efficiency]}}
    Under Game~\ref{game:aq}, given that $\A$ issues queries computable in $\poly(\log N)$ time, $\M_\aqals$ runs in time $\poly(\eps^{-1},\rho,\log N,M,T,\log(1/\del))$.
    \label{thm:AQALS-CC}
\end{restatable}

More importantly, we state the accuracy and sample complexity guarantees for $\M_\aqals$ in the following theorem, which is sketched in \S\ref{sec:analysis} and proved in Appendix~\ref{app:main-proof}.

\begin{restatable}{theorem}{thmAQALSSC}\emph{\textbf{[Sample Efficiency]}}
    If $\A$ almost surely chooses a distribution $\D$ for which some surrogate $\D'$ satisfies~\eqref{eq:L2-condition} with respect to its query family $\Q^\A$, then $\M_\aqals$, run with input $\rho$, is $(\eps,\del)$-accurate with respect to $\A$, with any sample size $M\geq C\big(\rho^2\eps^{-4}\log(T/\eps)+\eps^{-2}\log(1/\del)\big)$, where $C$ is a universal constant.
    \label{thm:AQALS-SC}
\end{restatable}

\begin{corollary}\label{cor:dense}\emph{\textbf{[Dense and Pseudodense Distributions]}}
If $\A$ is an $(\eps/16,\lam,\D_g)$-pseudodense challenger, $\M_\aqals$ with $\rho=\lam^{-1/2}$ is $(\eps,\del)$-accurate for $M\geq C\big(\lam^{-1}\eps^{-4}\log(T/\eps)+\eps^{-2}\log(1/\del)\big)$.
\end{corollary}
The corollary follows from $\expec_{\D_g}[w^2]\leq1/\lam$. The logarithmic dependence on $T$ is necessary even for non-adaptive queries (Appendix~\ref{app:lower-bound}).
\begin{restatable}{proposition}{propLowerBound}\label{prop:lower-bound}\emph{\textbf{[Lower Bound]}}
For $T\geq2$ and $0<\eps\leq1/8$, there is a family of $2/3$-dense distributions relative to the uniform prior on $\{-1,1\}^T$ such that answering the $T$ fixed coordinate queries simultaneously within $\eps$, with probability at least $2/3$, requires $M=\Omega(\eps^{-2}\log T)$ samples.
\end{restatable}

\paragraph{Sample-efficient Implementation of Distribution-specific SQ Oracle.}
Our ADA mechanism also yields a sample-efficient implementation of a distribution-specific SQ oracle. As discussed in \S\ref{subsec:dsq}, $\D_{h,f}$ is $1/2$-dense relative to the prior with independent uniform labels. Applying Corollary~\ref{cor:dense} to the Booleanized queries gives an implementation that reuses one labeled sample (Appendix~\ref{app:proof:thm:AQALS-DSQ}).

\begin{restatable}{theorem}{thmAQALSDSQ}\emph{\textbf{[Sample-efficient DSQ Oracle from $\M_\aqals$]}}
    Fix a feature distribution $\D_h$ over a finite domain $\X$ generated by an efficient sampler $h$, and an unknown labeling function $f:\X\to\{-1,1\}$. Let $\D_{h,f}$ denote the induced distribution over $(X,Y)$ given by $X\sim\D_h$ and $Y=f(X)$. Then there exists an implementation of the distribution-specific SQ oracle $\DSQ_h(f,\eps)$ with the following guarantee: for any analyst issuing a sequence of $T$ adaptively chosen queries $\phi_1,\ldots,\phi_T$ with $\phi_t:\X\times\{-1,1\}\to[-1,1]$, the implementation answers all queries simultaneously to within $\eps$ accuracy with probability at least $1-\del$, using
    \begin{align*}
        M\in\O\big(\eps^{-4}\log(BT/\eps)+\eps^{-2}\log(1/\del)\big),\quad B=\Theta(\eps^{-2}\log(T/\del))
    \end{align*}
    i.i.d.\ samples from $\D_{h,f}$, reused for all calls. If $h$ and each $\phi_t$ can be evaluated in $\poly(\log N)$ time, the runtime is $\poly(\log N,\eps^{-1},T,\log(1/\del))$.
    \label{thm:AQALS-DSQ}
\end{restatable}

\section{Sample Complexity Analysis}\label{sec:analysis}
In this section, we analyze the sample complexity of $\M_\aqals$ as stated in Theorem~\ref{thm:AQALS-SC}.

Our analysis has two parts. First, the bounded update transcript gives uniform concentration of $q_t(S)$ around $q_t(\D)$, even though the queries are adaptive. Second, the regret bound for OGD shows that the update budget is sufficient, so every answer remains close to $q_t(S)$.

\subsection{Proof Sketch of Theorem~\ref{thm:AQALS-SC}}
Here we provide a proof sketch of our main result, Theorem~\ref{thm:AQALS-SC}, which establishes the sample complexity of $\M_{\aqals}$. The full proof appears in Appendix~\ref{app:main-proof}.

\paragraph{Step 1. Reconstructing queries from succinct transcript.}
A central component of our analysis is the procedure $\replay\big(Z; \A, \omega\big)$, described in Alg.~\ref{alg:replay-app} in Appendix~\ref{app:proofs}%
, which ``replays'' all queries issued by $\A$ during the execution of $\aq_{N,M,T}[\A \rightleftharpoons \M_\aqals]$, when both $\A$ and $\M_\aqals$ are run with \emph{fixed} internal randomness $\omega_\A$ and $\omega_\M$, respectively. The main input to $\replay$ is a set of key-value pairs $Z \coloneq \cbra{(t_1, (a_{t_1}, s_{t_1})), \dots}$, where each key $t_i$ denotes a round in which $\G$ was updated, and the corresponding value $(a_{t_i}, s_{t_i})$ is the answer given to $\A$ and the sign of the update, $s_{t_i} \coloneq \sign(\uhat_{t_i} - \kappa q_{t_i}(S))$, as recorded in Line 8 of $\M_\aqals$ (Alg.~\ref{alg:Maqals}) at round $t_i$. We refer to $Z$ as the update transcript.

Given fixed randomness $\omega \coloneq (\omega_\A, \omega_\M)$, samples $S$ and the analyst’s algorithm $\A$, by construction, there exists a transcript $Z(\omega,S)$ such that $\replay(Z(\omega,S); \A, \omega)$ faithfully reproduces all queries issued by $\A$, assuming $\A$ and $\M_\aqals$ are executed with frozen randomness $\omega_\A$ and $\omega_\M$, respectively. This can be interpreted as simulating the behavior of $\M_\aqals$ from the transcript alone: each $a_t$ constructed by $\replay$ equals the one produced by $\M_\aqals$ with the same coins $\omega_\M$.

The hard cap bounds the number of updates by $K_0=\floor{128/(\kappa\eps)^2}+1$ on every execution. Each entry consists of a round, one of $\O(1/\eps)$ rounded answers, and a sign. Hence $\log|\Z|=\O(K_0\log(T/\eps))$. For fixed random coins, each replayed query is independent of $S$, so Hoeffding's inequality and a union bound over $T|\Z|$ queries give the required uniform concentration with
\[
M=\O\!\left(\frac{\log|\Z|+\log(T/\del)}{\eps^2}\right)
 =\O\!\left(\rho^2\eps^{-4}\log(T/\eps)+\eps^{-2}\log(1/\del)\right).
\]
The actual query sequence is one of these replays. Lemma~\ref{lem:qSqD} gives the full argument; it uses the hard cap without assuming that the execution is accurate.

\paragraph{Step 2. $\M_\aqals$ outputs $a_t \approx q_t(S)$.}
Assume that all empirical means are $\eps/8$-close to their population values and all inner-product estimates are $\kappa\eps/8$-accurate. On every actual update round, the signed query gives a subgradient of $f_t(x)=\left|\innp{\tq_t/\sqrt N,x}-\kappa q_t(\D)\right|$, and $f_t(x_t)>\kappa\eps/4$. The surrogate in~\eqref{eq:L2-condition} gives a comparator $x^\star_r=\kappa w(g(r))/\sqrt N$ with $\|x^\star\|\leq1$ and $f_t(x^\star)\leq\kappa\eps/16$. By the regret bound for unconstrained OGD, there are at most $128/(\kappa\eps)^2$ actual update rounds, so the budget cannot be exhausted. Each answer is therefore either $q_t(S)$ rounded to the nearest multiple of $\eps/2$ or $\uhat_t/\kappa$ with $|\uhat_t/\kappa-q_t(S)|\leq\eps/2$. Clipping preserves this error bound, and sample concentration then gives population accuracy (Appendix~\ref{app:main-proof}).

\section{Privacy Guarantee}\label{sec:privacy}
In this section, we establish the privacy-preserving properties of $\M_\aqals$. Although $\M_\aqals$ does not satisfy stronger guarantees such as differential privacy~\citep{dwork2014algorithmic}, we show that it enjoys a form of meaningful privacy protection known as Predicate Singling Out (PSO) security~\citep{cohen2020towards}, under the parameter conditions discussed after Theorem~\ref{thm:AQALS-PSO}. This notion was introduced as a formal counterpart to the concept of ``singling out'' that appears in the European Union’s General Data Protection Regulation (GDPR).

\paragraph{Predicate Singling Out (PSO).}
The main privacy leakage mode we consider is Predicate Singling Out (PSO). Specifically, given a dataset $S\in\I^M$, a PSO adversary attempts to produce a predicate $p : \I \rightarrow \{0,1\}$ that \emph{isolates} a single entry in $S$. Formally, this is captured by the \emph{row isolation} event:

\begin{definition}[Row isolation]
    Let $S=(I_1,\ldots,I_M)\in\I^M$ and let $p:\I\to\{0,1\}$. We say that $p$ \emph{isolates} a row in $S$ if there is a unique index $j\in[M]$ such that $p(I_j)=1$, or equivalently, if $p(S)=1/M$. We denote this event by $\iso(p, S)$.
\end{definition}

\begin{game}[t]
\caption{$\pso_{M,T,\P}[\A \rightleftharpoons \M]$: Predicate Singling Out Game} \label{game:pso}
\DontPrintSemicolon
A multiset $S$ of $M$ i.i.d. samples from $\D$ is given to $\M$ but not to $\A$.\\
\vspace{0.5em}
\For{$t \in [T]$}{
    $\A$ chooses a query $q_t \in \{-1,1\}^{\I}$ given $(q_1, a_1, \dots, q_{t-1}, a_{t-1})$.\\
    $\M$ outputs an answer $a_t \in [-1,1]$ given $(q_1, a_1, \dots, q_{t-1}, a_{t-1}, q_t)$.\\
}
$\A$ outputs a predicate $p:\I\to\{0,1\}$.
\end{game}

We now define PSO security in the context of the adaptive data analysis. Let $\P \subseteq \{0,1\}^{\I}$ be a class of admissible predicates. PSO security is formalized via a game between the ADA mechanism $\M$ and an adversary (or analyst) $\A$, as depicted in Game~\ref{game:pso}. As in the adaptive query (AQ) game (Game~\ref{game:aq}), we allow $\A$ to have full power: not only does $\A$ know the data distribution $\D$, but may also choose $\D$ arbitrarily. As before, any guarantee established against such a powerful $\A$ automatically applies to more realistic settings where $\D$ is \emph{unknown} to both $\M$ and $\A$. In fact, the original definition and analysis of PSO security by~\cite{cohen2020towards} allow for adversaries with this level of power.

We define the adversary’s success probability in the PSO game as:
\begin{align*}
    \succ^{\A,\M}_{\P,\D}(M,T)\coloneq \Pr_{\pso_{M,T,\P}[\A \rightleftharpoons \M]}\big[\, \iso(p,S)\text{ and } p\in\P\, \big].
\end{align*}
Following~\cite{cohen2020towards}, fix threshold functions $0\leq w_\low(M)\leq1/M\leq w_\high(M)\leq1$, and define $\P_\low=\{p:\I\to\{0,1\}:p(\D)\leq w_\low(M)\}$ and $\P_\high=\{p:\I\to\{0,1\}:p(\D)\geq w_\high(M)\}$. We consider $w_\low\in\negl(M)$ and $w_\high\in\omega(\log M/M)$, where $\negl(M)=\bigcap_{c\in\mathbb N}o(M^{-c})$.

For a fixed predicate of mass $\mu=p(\D)$ independent of $S$, the isolation probability is $M\mu(1-\mu)^{M-1}$. At $\mu=1/M$, this approaches $1/e$, so even without any output, a predicate of weight $1/M$ isolates a row with constant probability. For $\P_\low$, the baseline is at most $Mw_\low(M)$; for $\P_\high$, it is at most $M e^{-(M-1)w_\high(M)}$. Both bounds are negligible. The low-weight regime concerns exceptionally rare properties, whereas in the high-weight regime exactly one occurrence is unlikely because many records are expected to satisfy the predicate.

We now state the PSO security guarantee of $\M_\aqals$, with the full proof deferred to Appendix~\ref{app:proof:thm:AQALS-PSO}.
\begin{restatable}{theorem}{thmAQALSPSO}\emph{\textbf{[PSO security guarantee]}}
    Fix $w_\low\in\negl(M)$ and $w_\high\in\omega(\log M/M)$ as above. For each $\P\in\{\P_\low,\P_\high\}$, there is a function $\delta_\P(M)\in\negl(M)$, depending only on the corresponding threshold, such that for every $\D$ and every adversary $\A$,
    \[
    \succ^{\A,\M_\aqals}_{\P,\D}(M,T)
    \leq (T/\eps)^{\O(\rho^2\eps^{-2})}\delta_\P(M).
    \]
    This bound holds for any input norm bound $\rho\geq1$, without a density assumption on $\D$.

    \label{thm:AQALS-PSO}
\end{restatable}
If the number of possible transcripts is polynomial in $M$, the bound is negligible for both predicate classes. This is the case, for example, when $\rho$ and $\eps$ are constant and $T=\poly(M)$. For dense distributions, with $\rho=\lam^{-1/2}$, the exponent is $\O(\lam^{-1}\eps^{-2})$. The accuracy theorem alone, which permits $M=\O(\log T)$, does not ensure that the number of transcripts is polynomial in $M$. Thus the bound implies negligible PSO success under this additional condition, or when the baseline decays fast enough to offset the number of transcripts. The bound is also quantitatively conservative, because the constant in the exponent is large ($K_0\geq513$ when $\eps\leq1/2$). Finally, the guarantee excludes all predicate weights strictly between $w_\low(M)$ and $w_\high(M)$, including $1/M$; for such weights the baseline itself can be non-negligible.

\section{Discussion and Limitations}\label{sec:discussion}
Our results show that, with a shared efficient prior sampler, a computationally efficient ADA mechanism attains sample complexity logarithmic in $T$ whenever the target, or a surrogate matching the analyst's queries, has a likelihood ratio with bounded $L_2(\D_g)$ norm. Since the number of lazy OGD updates is capped, the update transcript is always succinct and yields adaptive concentration, the distribution-specific SQ implementation, and the PSO bound. The norm bound ensures that the cap is never reached when the empirical means and inner-product estimates are accurate.

The shared efficient sampler and the surrogate assumption are substantive requirements. A public sample alone does not ensure density outside its support, and pseudodensity depends on the analyst's full query family. It remains open to characterize the weakest condition relative to a prior under which efficient mechanisms retain logarithmic dependence on $T$. There is also a gap between the $\eps^{-4}$ upper bound and the $\eps^{-2}$ lower bound. Computational efficiency here means polynomial runtime, whose dependence on the accuracy and number of queries may still be large. The PSO result covers specified ranges of predicate weight and yields negligible success under additional parameter conditions. It does not provide differential privacy or unrestricted protection against identification.

\begin{ack}
This research received no funding support. The author declares no competing interests.
\end{ack}

\bibliography{ref}
\bibliographystyle{alpha}

\newpage
\appendix
\section{Proof of Theorem~\ref{thm:AQALS-DSQ}}\label{app:proof:thm:AQALS-DSQ}
In this appendix, we provide a proof of Theorem~\ref{thm:AQALS-DSQ}, which states that $\M_\aqals$ yields a sample-efficient implementation of a distribution-specific statistical query oracle. For convenience, we restate the theorem below.

\thmAQALSDSQ*

\begin{proof}[Proof of Theorem~\ref{thm:AQALS-DSQ}]
As discussed above, $\D_{h,f}$ is $(1/2)$-dense with respect to the prior $\D_g$ that samples $X\sim\D_h$ together with an independent uniform label in $\{-1,1\}$. To handle $[-1,1]$-valued queries, for each $\phi_t$ draw $B=\ceil{8\eps^{-2}\log(4T/\del)}$ independent thresholds $U_{t,j}\sim\operatorname{Unif}[-1,1]$ and define
\[
q_{t,j}(x,y)=\sign\big(\phi_t(x,y)-U_{t,j}\big).
\]
Each threshold is drawn once and held fixed whenever its query is evaluated. Conditional on $\phi_t$ and the history preceding it, the values $q_{t,j}(\D_{h,f})$ are independent, lie in $[-1,1]$, and have mean $\phi_t(\D_{h,f})$. Hoeffding's inequality bounds the conditional failure probability by $2e^{-B\eps^2/8}\leq\del/(2T)$. Averaging over $\phi_t$ and the history and taking a union bound over $t\in[T]$, with probability at least $1-\del/2$ we have
\begin{align*}
    \left|\frac{1}{B}\sum_{j\in[B]} q_{t,j}(\D_{h,f})-\phi_t(\D_{h,f})\right|\le \eps/2
    \qquad\text{for all }t\in[T].
    \numberthis\label{eq:batch-q-phi}
\end{align*}
We then run $\M_\aqals$ once, using the same fixed labeled sample $S\sim\D_{h,f}^M$, on the resulting sequence of $T' = BT$ Boolean queries with accuracy parameter $\eps/2$ and failure probability $\del/2$, which requires $M\in\O(\eps^{-4}\log(BT/\eps)+\eps^{-2}\log(1/\del))=\O(\eps^{-4}\log(T/(\eps\del)))$ samples, as $B=\Theta(\eps^{-2}\log(T/\del))$. Drawing the thresholds requires no additional labeled examples. By Corollary~\ref{cor:dense} with $\lam=1/2$ (so $\rho=\sqrt2$ and $\kappa=1/\sqrt2$), with probability at least $1-\del/2$, $\M_\aqals$ answers every $q_{t,j}$ to within $\eps/2$. 

Let $\widehat{\phi}_t$ denote the average of the $B$ answers corresponding to $(q_{t,j})_{j\in [B]}$. This averaging introduces no additional error and approximates $\frac{1}{B}\sum_{j\in[B]} q_{t,j}(\D_{h,f})$ within $\eps/2$ error. Combining this with~\eqref{eq:batch-q-phi}, we conclude that $|\widehat{\phi}_t-\phi_t(\D_{h,f})|\le \eps$ for all $t\in[T]$ with probability at least $1-\del$.
\end{proof}
\section{Proof of Sample Complexity of \texorpdfstring{$\M_\aqals$}{AQA-OGD}}\label{app:proofs}
In this appendix, we provide a complete proof of our main result, Theorem~\ref{thm:AQALS-SC}, which establishes the sample complexity of $\M_\aqals$.

\subsection{Proof of Theorem~\ref{thm:AQALS-SC}}\label{app:main-proof}
\paragraph{Step 1. Proving $|q_t(S)-q_t(\D)|\in\O(\eps)$.}
In this step, we prove the following lemma, which establishes the sample complexity required to guarantee uniform concentration between $q_t(S)$ and $q_t(\D)$ over all $t \in [T]$.

{\LinesNotNumbered%
\begin{algorithm2e}[t]
\caption{$\replay(Z;\A,\omega)$: Reconstruct queries and answers $(q_1,a_1,\dots,q_T,a_T)$ without samples $S$}
\label{alg:replay-app}
\SetKwInOut{Input}{Input}
\SetKwInOut{Output}{Output}
\SetKwComment{Comment}{\hfill\color{blue}$\triangleright$~\color{blue}}{}
\DontPrintSemicolon

\textbf{Input:}
\vspace{0.5em}
\begin{itemize}[leftmargin=2em]
    \item Update transcript $Z\coloneq \cbra{(t_1, (a_{t_1}, s_{t_1})), \dots}$, represented as key-value pairs, where each key $t_i$ denotes a round in which $\G$ was updated, and the corresponding value $(a_{t_i}, s_{t_i})$ records the answer given to $\A$ and the sign of the update at round $t_i$.
    
    \item Adversary’s algorithm $\A$ and random bits $\omega=(\omega_\A, \omega_\M)$, where $\omega_\A \in \{0,1\}^\star$ are used by $\A$, and $\omega_\M \in \{0,1\}^\star$ are used by $\M_\aqals$ during calls to $\estimate$ (Alg.~\ref{alg:estimate}).
\end{itemize}

The parameters $g,\rho,\eps,\del,T$ are the same as in Alg.~\ref{alg:Maqals}.\\
\textbf{Output:} A tuple of reconstructed queries and answers $(q_1, a_1, \dots, q_T, a_T)$.\\
\vspace{0.5em}
$\G \gets \emptyset$

\For{$t\in [T]$}{
    \vspace{0.25em}
    $q_t \gets \A^{\omega_\A}(q_1, a_1, \dots, q_{t-1}, a_{t-1})$; set $\tq_t=q_t\circ g$ \\
    \vspace{0.25em}
    \Comment*{{\rmfamily\color{blue}Run $\A$ with randomness $\omega_\A$ and previous query-answer history.}}
    \vspace{0.5em}
    $\uhat_t \gets \estimate^{\omega_\M}(\tq_t,\G ;\kappa \eps / 8, \del/2T)$\\
    \vspace{0.25em}
    \Comment*{{\rmfamily\color{blue} Run $\estimate$ using $\omega_\M$, reading from where the last $\estimate$ call left off.}}
    \vspace{0.50em}
    \If{$t \in Z.\keys$}{
        \vspace{0.5em}
        $(a_t, s_t) \gets Z.\get(t)$ \Comment*{{\rmfamily\color{blue}Get the entry with key $t$.}}
        \vspace{0.5em}
        $\G \gets \G \cup \big\{ -\kappa\eps s_t \tq_t / 16 \big\}$ \Comment*{{\rmfamily\color{blue}Update $\G$.}}
        \vspace{0.5em}
        $Z.\remove(t)$ \Comment*{{\rmfamily\color{blue}Remove the entry with key $t$.}}
        \vspace{0.5em}
    }
  \Else{
    $a_t \gets \min\{1,\max\{-1,\uhat_t/\kappa\}\}$
  }
}
\Return{$(q_1, a_1, \dots, q_T, a_T)$}
\end{algorithm2e}}

\begin{restatable}{lemma}{lemConcentration}\textup{[Uniform Concentration]}
    Let $\M_\aqals$ be the ADA mechanism given in Alg.~\ref{alg:Maqals}. Then, for any adversary $\A$ and $\del>0$,
    \begin{align*}
        \Pr_{\aq_{N,M,T}[\A \rightleftharpoons \M_\aqals]}\Big[\,\forall t\in [T],\ \left| q_t(S)-q_t(\D) \right|\leq\frac{\eps}{8}\,\Big]\geq 1-\frac{\del}{2},
    \end{align*}
    whenever $M\geq C\big(\rho^2\eps^{-4}\log(T/\eps)+\eps^{-2}\log(1/\del)\big)$ for a universal constant $C$.
    \label{lem:qSqD}
\end{restatable}

\begin{proof}[Proof of Lemma~\ref{lem:qSqD}]
    We prove the lemma by explicitly constructing an algorithm, $\replay$ given in Alg.~\ref{alg:replay-app}, that reconstructs queries $(q_1,\dots, q_T)$ asked by $\A$ during game $\aq_{N,M,T}[\A \rightleftharpoons \M_\aqals]$, from a succinct transcript.
    
    A main input to $\replay$ is the update transcript, $Z\coloneq\{(t_1,(a_{t_1},s_{t_1})),\dots\}$ represented by a set of key-value pairs $(t_i,(a_{t_i},s_{t_i}))$ where the key $t_i \in [T]$ is a round in which $\G$ was updated and $(a_{t_i},s_{t_i})$ is the pair of the answer provided to $\A$ and the sign of the update, $\sign(\uhat_{t_i}-\kappa q_{t_i}(S))$, at round $t_i$.

    Since the if-statement of $\M_\aqals$ allows an update only while $|\G|\leq 128(\kappa\eps)^{-2}$, $\M_\aqals$ updates $\G$ at most $K_0\coloneq \floor{128(\kappa\eps)^{-2}}+1$ times, and the number of entries in $Z$ is at most $K_0$. Each entry consists of a round in $[T]$, an answer $a_t$, which is $q_t(S)\in[-1,1]$ rounded to a multiple of $\eps/2$ and clipped to $[-1,1]$ and hence takes at most $4/\eps+2$ values, and a sign $s_t\in\{\pm 1\}$. Padding $Z$ with empty entries to $K_0$ entries, there are at most $\big((T+1)(8/\eps+4)\big)^{K_0}\leq(24T/\eps)^{K_0}$ possible choices for $Z$, using $\eps\leq 1$. Let the set of all such logs be denoted by $\Z$. The hard cap bounds this set on every execution, including those on which accuracy fails.
    
    Then, for any fixed random bits $\omega$ used by $\M_\aqals$ and $\A$ (including those used by $\A$ to select the distribution $\D$), we have \begin{align*}
        \Pr_{S\sim \D^M}\bigg[\forall (t,Z)\in[T]\times\Z,\ \left| q_t(S)-q_t(\D) \right|&\leq\sqrt{\frac{2\big(\log|\Z|+\log(4T/\del)\big)}{M}},\\
        &\hspace{-0.2in}\text{where}\ (q_t,a_t)_{t\in[T]} \gets \replay(Z;\A,\omega)\,\bigg]\geq 1-\frac{\del}{2}
        \numberthis\label{eq:replay-bound}
    \end{align*} 
    by applying the union bound over $[T]\times\Z$ to the Hoeffding bound $\Pr_{S\sim\D^M}[|q(S)-q(\D)|>\tau]\leq 2e^{-M\tau^2/2}$, which holds for each fixed query $q\in\{-1,1\}^{\I}$. The right-hand side of the inequality in~\eqref{eq:replay-bound} is at most $\eps/8$ whenever $M\geq 128\eps^{-2}\big(\log|\Z|+\log(4T/\del)\big)$, which holds under the condition of the lemma, as $\log|\Z|\leq K_0\log(24T/\eps)$.
    
    For each randomness $\omega$ and sample $S$, there is $Z(\omega, S) \in \Z$ such that $\replay(Z(\omega,S);\A,\omega)$ reconstructs the queries $(q_t)_{t \in [T]}$ asked in game $\aq_{N,M,T}[\A \rightleftharpoons \M_\aqals]$ when run with fixed $\omega$ and $S$. Hence, taking the expectation over $\omega$ in \eqref{eq:replay-bound}, this fact yields the bound stated in the lemma.
\end{proof}

\paragraph{Step 2. Finalizing the proof by showing $|a_t-q_t(S)|\in\O(\eps)$.}
With the high-probability concentration bound $|q_t(S) - q_t(\D)| \in \O(\eps)$ established in Step 1, we now complete the proof of Theorem~\ref{thm:AQALS-SC} by showing that $|a_t - q_t(S)| \in \O(\eps)$ for all $t \in [T]$. 

This follows from the fact that the number of updates made by Alg.~\ref{alg:Maqals} remains small, i.e., $|\G| \leq 128 /(\kappa\eps)^2$, as ensured by a regret argument. In particular, we obtain a contradiction by showing that exceeding this number of updates would violate the regret bound for OGD. Consequently, the condition $|\G| \leq 128 /(\kappa\eps)^2$ in the if-statement of Alg.~\ref{alg:Maqals} is redundant, and $\M_\aqals$ therefore outputs approximately correct answers.

Before proceeding, we formally state the regret bound for unconstrained OGD. While this result is standard (see, e.g.,~\cite{bubeck2015convex,hazan2016introduction,zinkevich2003online}), we include a self-contained proof in Appendix~\ref{app:technical-proofs} for completeness.

\begin{restatable}{lemma}{lemOGD}\textup{[Regret Bound for Unconstrained OGD]}
    For any finite ordered set $\T \coloneq \{t_1, \dots, t_{|\T|}\} \subset \nat$ such that $t_1 < \cdots < t_{|\T|}$, and for each $t \in \T$, let $f_t : \rr^N \rightarrow \rr$ be any convex loss function with unit-bounded subgradients, i.e., $\|\nabla f_t\| \leq 1$. Then, for any $\alpha > 0$ and $x^\star \in \bb^N$, \begin{align*} 
        \sum_{t \in \T} \big[f_t(x_t) - f_t(x^\star)\big] \leq \frac{1}{2\alpha} + \frac{\alpha |\T|}{2}, 
        \numberthis\label{eq:OGDRegret} 
    \end{align*} where $x_t$ is defined recursively as 
    \begin{align*} 
        x_{t_{i+1}} \coloneq x_{t_i} - \alpha \nabla f_{t_i}(x_{t_i})\quad\textup{and}\quad x_{t_1} \coloneq 0. 
        \numberthis\label{eq:OGDUpdateRule} 
    \end{align*} 
    \label{lem:OGD}
\end{restatable}
\paragraph{FTRL interpretation.}
Let $\T_\textup{mis}\subseteq[T]$ denote the rounds in which $\G$ is actually updated, and let $\gamma_u=s_u\tq_u/\sqrt N$ for $u\in\T_\textup{mis}$. With $\alpha=\kappa\eps/16$, the normalized sketch $x_t$ at round $t$ satisfies the identity
\begin{equation}
    x_t=-\alpha\sum_{u\in\T_\textup{mis},\,u<t}\gamma_u
       =\argmin_{x\in\rr^N}\left\{\sum_{u\in\T_\textup{mis},\,u<t}\innp{\gamma_u,x}
                                  +\frac{\|x\|^2}{2\alpha}\right\}.
    \label{eq:FTRL}
\end{equation}
This follows by differentiating the quadratic objective and is the standard constant-step OGD/FTRL equivalence~\citep[Section~3.1]{mcmahan2017survey}. The identity holds for the recorded signs on every execution. In the proof below, accurate inner-product estimates and sample concentration ensure that these signs give subgradients of the population absolute losses on update rounds. Thus the FTRL objective accumulates their linearizations; no constraint or projection onto $\bb^N$ is imposed on the iterates.

With all the necessary ingredients established so far, we now present the proof of Theorem~\ref{thm:AQALS-SC}, which we restate below for convenience.
\thmAQALSSC*
\begin{proof}[Proof of Theorem~\ref{thm:AQALS-SC}]\label{proof:thm:AQALS-SC}
    Recall that $\T_\textup{mis}\subseteq[T]$ denotes the rounds in which $\G$ is actually updated. We first bound the size of this set using OGD on these rounds only. If $|\T_\textup{mis}|\leq128/(\kappa\eps)^2$, the update budget cannot be exhausted: reaching a state above the cap would require $\floor{128/(\kappa\eps)^2}+1$ actual updates. Thus every threshold violation causes an update, and every other round has $|\uhat_t/\kappa-q_t(S)|\leq\eps/2$. Rounding introduces at most $\eps/4$ error and clipping cannot increase error relative to $q_t(S)\in[-1,1]$.

    Let $\mathsf{E}$ be the event from Lemma~\ref{lem:qSqD}, that is $\forall t\in[T],\ |q_t(S)-q_t(\D)|\leq\eps/8$, and $\mathsf{F}$ be the event that $\estimate$ produces a $\kappa\eps/8$-accurate estimate for all $t \in [T]$. We show that when both events occur (i.e., $\mathsf{E}\cap\mathsf{F}$), we have $|\T_\textup{mis}| \leq \frac{128}{(\kappa\eps)^2}$, hence $|a_t-q_t(S)|\leq\eps/2$ for all $t \in [T]$. Moreover, under $\mathsf{E}$ we have $|q_t(S) - q_t(\D)| \leq \eps/8$ for all $t \in [T]$. Hence, under $\mathsf{E}\cap\mathsf{F}$, we have $|a_t-q_t(\D)|\leq\eps$ for all $t \in [T]$, which is the theorem's claim.
    
    Therefore, it suffices to show the following two claims:
    
    \textbf{Claim 1.} $\Pr_{\aq_{N,M,T}[\A \rightleftharpoons \M_\aqals]}[\mathsf{E}\cap \mathsf{F}]\geq 1-\del$.
    
    \textbf{Claim 2.} Given that $\mathsf{E}\cap \mathsf{F}$ occurs and $\D$ admits a surrogate satisfying~\eqref{eq:L2-condition}, $|\T_\textup{mis}|\leq\frac{128}{(\kappa\eps)^2}$.

    \textit{Proof of Claim 1.} Lemma~\ref{lem:qSqD} shows that $\Pr_{\aq_{N,M,T}[\A \rightleftharpoons \M_\aqals]}[\mathsf{E}] \geq 1 - \del/2$. Hence, it is enough to show that 
    \begin{align*} 
        \mathsf{F} \coloneq \cbra{\forall t\in[T],\ \left|\uhat_t-\frac{1}{N}\innp{\tq_t,y_t}\right|\leq \frac{\kappa\eps}{8}},
    \end{align*}
    occurs with probability at least $1 - \del/2$ under $\aq_{N,M,T}[\A \rightleftharpoons \M_\aqals]$, where $\uhat_t$ is the output of $\estimate(\tq_t, \G; \kappa\eps/8, \del/2T)$ and $y_t\coloneq\sum_{v\in\G}v$ for the multiset $\G$ at the start of round $t$.
    Since the randomness used by $\estimate$ to estimate $\innp{\tq_t,y_t}/N$ is independent from $q_t$ and $\G$, we can apply the union bound over Lemma~\ref{lem:Estimate}, which yields
    \begin{align*}
        \Pr_{\aq_{N,M,T}[\A \rightleftharpoons \M_\aqals]}[\,\mathsf{F}\,] \geq 1 - \frac{\del}{2}.
    \end{align*}
    \qedsymbol\\
    \\
    \textit{Proof of Claim 2.} For each $t \in \T_\textup{mis}$, define $f_t : \rr^N \rightarrow \mathbb{R}$ as 
    \begin{align*} 
        f_t(x) \coloneq \left| h_t(x) \right|,\quad h_t(x)\coloneq\innp{\frac{\tq_t}{\sqrt{N}},x} - \kappa q_t(\D).
    \end{align*} 
    so that $f_t$ is convex, and its sub-gradient satisfies $\|\nabla f_t(x)\| = \|\sign(h_t(x)) \tq_t / \sqrt{N}\| = 1$, since $\tq_t \in \{\pm 1\}^N$.

    For each $t \in \T_\textup{mis}$, let $x_t \coloneq y_t / \sqrt{N}$, where $y_t \coloneq \sum_{v \in \G} v$ is computed using the gradient set $\G$ at round $t$. Notice that, by construction, $\G$ consists of vectors $-\kappa\eps s_t \tq_t / 16$, where $s_t \coloneq \sign(\uhat_t - \kappa q_t(S))$ and $|\uhat_t - \kappa q_t(S)| > \kappa\eps/2$. Now, we list two key implications under $\mathsf{E} \cap \mathsf{F}$: 
    \begin{itemize}[leftmargin=2em]
        \item Given that $\mathsf{E} \cap \mathsf{F}$ has occurred, $|\uhat_t - \kappa q_t(S)| > \kappa\eps/2$ implies 
        \begin{align*}
          \sign(\uhat_t - \kappa q_t(S)) = \sign\rbra{\innp{\frac{\tq_t}{\sqrt{N}},x_t} - \kappa q_t(\D)} = \sign(h_t(x_t)),
        \end{align*}
        as $\left|\uhat_t-\innp{\tq_t,y_t}/N\right| = \left|\uhat_t-\innp{\tq_t/\sqrt{N},x_t}\right| \leq \kappa\eps/8$ under $\mathsf{F}$ and $|q_t(S)-q_t(\D)|\leq \eps/8$ under $\mathsf{E}$.
        
        Hence, $\G$ consists of vectors of the form $-\alpha\sqrt{N}\nabla f_t(x_t)$ with $\alpha=\kappa\eps/16$, and thus $x_t\coloneq\sum_{v\in\G}v/\sqrt{N}$ follows the OGD update rule given in~\eqref{eq:OGDUpdateRule}, with learning rate $\kappa\eps/16$.
        
        \item Similarly, under $\mathsf{E} \cap \mathsf{F}$, $|\uhat_t - \kappa q_t(S)| > \kappa\eps/2$ implies $f_t(x_t) > \kappa\eps/4$.
    \end{itemize}
    From the above implications together with the OGD regret bound~\eqref{eq:OGDRegret}, we observe the following: For any $x^\star\in \bb^N$,
    \begin{align*}
        \frac{\kappa\eps}{4}|\T_\textup{mis}|-\sum_{t\in\T_\textup{mis}}f_t(x^\star)&\leq\sum_{t\in\T_\textup{mis}}\big[ f_t(x_t)-f_t(x^\star)\big]\leq \frac{8}{\kappa\eps} + \frac{\kappa\eps}{32}|\T_\textup{mis}| \\
        &\hspace{-0.5in}\implies \frac{\kappa\eps|\T_\textup{mis}|}{8}-\sum_{t\in\T_\textup{mis}}f_t(x^\star) \leq \frac{8}{\kappa\eps}.
    \end{align*}
    Hence, provided that there is some $x^\star\in\bb^N$ such that 
    \begin{align*}
        \sum_{t\in\T_\textup{mis}}f_t(x^\star) \leq \frac{\kappa\eps|\T_\textup{mis}|}{16},
    \end{align*}
    we obtain the desired upper bound on the mistake set: $|\T_\textup{mis}| \leq 128/(\kappa\eps)^2$. Indeed, this follows from~\eqref{eq:L2-condition}, as we show below:

    \textbf{Claim 3.} Under~\eqref{eq:L2-condition}, there is $x^\star\in\bb^N$ such that $\sum_{t\in\T_\textup{mis}}f_t(x^\star)\leq\kappa\eps|\T_\textup{mis}|/16$.

    \textit{Proof of Claim 3.} Choose the surrogate $\D'$ from~\eqref{eq:L2-condition}, and let $w=d\D'/d\D_g$. Define the comparator in seed space by
    \[
        x^\star_r=\frac{\kappa w(g(r))}{\sqrt N},\qquad r\in[N].
    \]
    Then
    \begin{align*}
        \|x^\star\|^2
        &=\frac{\kappa^2}{N}\sum_{r=1}^N w(g(r))^2
          =\kappa^2\expec_{\D_g}[w^2]\leq1,\\
        \innp{\frac{\tq_t}{\sqrt N},x^\star}
        &=\kappa\expec_{I\sim\D_g}[q_t(I)w(I)]
          =\kappa q_t(\D').
    \end{align*}
    Since every issued query belongs to $\Q^\A$, we have $f_t(x^\star)=\kappa|q_t(\D')-q_t(\D)|\leq\kappa\eps/16$. Summing over the actual update rounds proves the claim. The comparator is used only in the analysis and need not be computed by the mechanism. \qedsymbol

    Therefore, we conclude the proof of Claim 2. \qedsymbol 
    
    \noindent Claims 1 and 2 complete the proof of Theorem~\ref{thm:AQALS-SC}.
\end{proof}

\subsection{Proofs for Technical Lemmas}\label{app:technical-proofs}
In this subsection, we collect technical lemmas and their proofs used in the proof of Theorem~\ref{thm:AQALS-SC}. We first restate, for convenience, the regret bound for unconstrained OGD and provide its proof~\citep{zinkevich2003online} for completeness.
\lemOGD*
\begin{proof}[Proof of Lemma~\ref{lem:OGD}]
    Fix any $x^\star\in\bb^N$. We begin by observing that
    \begin{align*}
        \|x_{t_{i+1}}-x^\star\|^2&=\|x_{t_i}-\alpha\nabla f_{t_i}(x_{t_i})-x^\star \|^2\\
        &=\|x_{t_i}-x^\star\|^2 -2 \alpha\innp{\nabla f_{t_i}(x_{t_i}),x_{t_i}-x^\star} + \alpha^2\| \nabla f_{t_i}(x_{t_i})\|^2.
    \end{align*}
    As $\|\nabla f_{t_i}(x_{t_i})\|\leq 1$, rearranging the above, we have that
    \begin{align*}
        \innp{\nabla f_{t_i}(x_{t_i}),x_{t_i}-x^\star}\leq\frac{\|x_{t_i}-x^\star\|^2-\|x_{t_{i+1}}-x^\star\|^2}{2\alpha}+\frac{\alpha}{2}.
    \end{align*}
    Since $f_t$s are convex, $f_{t_i}(x_{t_i})-f_{t_i}(x^\star)\leq \innp{\nabla f_{t_i}(x_{t_i}),x_{t_i}-x^\star}$ for all $i$. Thus summing the above over all $i=1,\dots, |\T|$, 
    \begin{align*}
        \sum_{t\in\T} f_t(x_t)-f_t(x^\star)&\leq \sum_{i=1}^{|\T|}\innp{\nabla f_{t_i}(x_{t_i}),x_{t_i}-x^\star}\\
        &\leq \frac{\|x_{t_1}-x^\star\|^2}{2\alpha}+\frac{\alpha|\T|}{2}\\
        &\leq \frac{1}{2\alpha}+ \frac{\alpha|\T|}{2},
    \end{align*}
    as $x_{t_1}=0$.
\end{proof}

Finally, we provide a simple proof of correctness for the subroutine $\estimate$.
\begin{restatable}{lemma}{lemEstimate}\textup{[Correctness of $\estimate$]}
    Suppose $|\G|\leq T$ and every $v\in\G$ has the form $v=\alpha q'$ with $|\alpha|\leq1$ and $q'\in\{-1,1\}^N$, as in Alg.~\ref{alg:Maqals}. The output of $\estimate(q,\G;\eps,\del)$ (Alg.~\ref{alg:estimate}) lies in the interval $\big[\innp{q, y}/N-\eps,\innp{q, y}/N+\eps\big]$ where $y\coloneq\sum_{v\in\G}v$, with probability at least $1-\del$.
    \label{lem:Estimate}
\end{restatable}

\begin{proof}[Proof of Lemma~\ref{lem:Estimate}]
The estimand is
\begin{align*}
    \frac{1}{N}\innp{q, v} = \expec_{i \sim U_{[N]}}[q_i v_i].
\end{align*}
Since every vector $v \in \G$ can be written as $v = \alpha q'$ for some $|\alpha| \leq 1$ and $q' \in \{-1,1\}^N$, and $q \in \{-1,1\}^N$, we have $|q_i v_i| \leq 1$ for all $i \in [N]$. With the stated sample size, Hoeffding’s inequality bounds the estimation error for each $v$ by $\eps/T$ except with probability $\del/T$. A union bound over at most $T$ vectors and summation of these errors gives total error at most $\eps$ with probability at least $1-\del$. The case $\G=\emptyset$ is exact.
\end{proof}

\section{Proof of Theorem~\ref{thm:AQALS-PSO}}\label{app:proof:thm:AQALS-PSO}
In this appendix, we provide a proof of Theorem~\ref{thm:AQALS-PSO}, which establishes the PSO security guarantee of $\M_\aqals$. For convenience, we restate the theorem below.

\thmAQALSPSO*

\begin{proof}[Proof of Theorem~\ref{thm:AQALS-PSO}]
    Fix the random strings used by $\A$ and $\M$, denoted by $\omega_\A$ and $\omega_\M$ respectively, and let $\omega \coloneq (\omega_\A, \omega_\M)$. Then, we may use $\replay(Z;\A,\omega)$ from Alg.~\ref{alg:replay-app} to simulate the interaction between $\A$ and $\M_\aqals$ and recover the predicate $p=\A^{\omega_\A}(q_1,a_1,\dots,q_T,a_T)$ that $\A$ outputs. By the same argument used in the proof of Lemma~\ref{lem:qSqD}, there exists a transcript $Z(\omega, S) \in \mathcal{Z}$ such that $\A^{\omega_\A}\big(\replay(Z(\omega, S); \A, \omega)\big)$ outputs the correct predicate $p$. Furthermore, the total number of possible transcripts satisfies $|\Z|\leq(24T/\eps)^{K_0}$ with $K_0=\floor{128(\kappa\eps)^{-2}}+1$, with $\kappa=1/\rho$ (see the proof of Lemma~\ref{lem:qSqD} in Appendix~\ref{app:main-proof}).
    
    For fixed $\omega$ and $Z$, the predicate $p_Z=\A^{\omega_\A}(\replay(Z;\A,\omega))$ is independent of $S$. Define
    \[
    \operatorname{base}_{\P,\D}(M)
    =\sup_{p\in\P} M p(\D)(1-p(\D))^{M-1},
    \]
    with value zero when $\P$ is empty. A union bound over all transcripts gives
    \begin{align*}
    \Pr_{S\sim\D^M}[\iso(p_{Z(\omega,S)},S),\ p_{Z(\omega,S)}\in\P]
    &\leq\sum_{Z\in\Z}\Pr_{S\sim\D^M}[\iso(p_Z,S),\ p_Z\in\P]\\
    &\leq|\Z|\operatorname{base}_{\P,\D}(M).
    \end{align*}
    Averaging over $\omega$ yields
    \[
    \succ^{\A,\M_\aqals}_{\P,\D}(M,T)
    \leq\min\{1,|\Z|\operatorname{base}_{\P,\D}(M)\}.
    \]
    For the fixed threshold families in \S\ref{sec:privacy},
    \begin{align*}
    \operatorname{base}_{\P_\low,\D}(M)&\leq M w_\low(M),\\
    \operatorname{base}_{\P_\high,\D}(M)&\leq M e^{-(M-1)w_\high(M)}.
    \end{align*}
    Write $\mu=p(\D)$. The first bound follows by dropping $(1-\mu)^{M-1}\leq1$. For the second, use $\mu\leq1$, $1-\mu\leq e^{-\mu}$, and $\mu\geq w_\high(M)$. Both bounds are negligible under the specified threshold assumptions, uniformly over $\D$. Taking these as $\delta_\P(M)$ and using $|\Z|\leq(24T/\eps)^{K_0}$ proves the theorem. Only the hard update cap is needed; neither sample accuracy nor the surrogate assumption is used.

\end{proof}

\section{A Lower Bound for Constant-Density Distributions}\label{app:lower-bound}
\propLowerBound*
\begin{proof}
Let $U$ be uniform on $\{-1,1\}^T$. For each $j\in[T]$ and $s\in\{-1,1\}$, define
\[
\D_{j,s}(x)=2^{-T}(1+4\eps s x_j).
\]
Since $\expec_U x_j=0$, this is normalized. For $\eps\leq1/8$, its likelihood ratio lies in $[1/2,3/2]$, so it is nonnegative and $2/3$-dense relative to $U$. For the fixed coordinate queries $q_k(x)=x_k$,
\[
q_k(\D_{j,s})=4\eps s\,\mathbf{1}\{k=j\}.
\]
If all answers have error at most $\eps$, the unique coordinate whose answer has magnitude greater than $2\eps$ identifies $j$, and its sign identifies $s$.

Let $\Theta=(j,s)$ be uniform on the $2T$ choices, and draw $S\sim\D_\Theta^M$. Writing $a=4\eps$ and using natural logarithms,
\begin{align*}
D_{\mathrm{KL}}(\D_{j,s}\|U)
&=\frac{(1+a)\log(1+a)+(1-a)\log(1-a)}2\\
&\leq\frac{(1+a)a-(1-a)a}2=a^2=16\eps^2,
\end{align*}
where the inequality uses $\log z\leq z-1$. If $\overline{P}$ denotes the mixture of the laws $\D_\Theta^M$, the relative-entropy identity gives
\begin{align*}
I(\Theta;S)
&=\expec_\Theta D_{\mathrm{KL}}(\D_\Theta^M\|U^M)
  -D_{\mathrm{KL}}(\overline{P}\|U^M)
 \leq16M\eps^2.
\end{align*}
The mechanism's random coins are independent of $(\Theta,S)$. Its answers therefore give an estimator of $\Theta$ with error probability at most $1/3$. Fano's inequality yields
\[
16M\eps^2\geq \frac23\log(2T)-\log2\geq\frac13\log T,
\]
where the last inequality uses $T\geq2$. Hence $M\geq(48\eps^2)^{-1}\log T$. This construction uses a domain of size $2^T$. Since the upper bound of Corollary~\ref{cor:dense} holds for every domain size, its dependence on $T$ cannot be improved; the gap between the upper bound's $\eps^{-4}$ and the lower bound's $\eps^{-2}$ dependence remains.
\end{proof}

\end{document}